\def\year{2019}\relax
\documentclass[letterpaper]{article} 
\usepackage{aaai19}  
\usepackage{times}  
\usepackage{helvet} 
\usepackage{courier}  
\usepackage[hyphens]{url}  
\usepackage{graphicx} 
\def\UrlFont{\rm}  

\usepackage{amsmath,amssymb,amsfonts}
\usepackage{xspace}
\usepackage{enumitem}
\usepackage{algorithm}
\usepackage{algorithmicx}
\usepackage[noend]{algpseudocode}
\usepackage{algpseudocode}
\usepackage{parskip}
\usepackage{amsthm}
\usepackage{subfig}

\algrenewcommand\algorithmicindent{1.0em}%

\newcommand{\calX}{\ensuremath{\mathcal{X}}\xspace}
\newcommand{\calL}{\ensuremath{\mathcal{L}}\xspace}
\newcommand{\calS}{\ensuremath{\mathcal{S}}\xspace}
\newcommand{\calR}{\ensuremath{\mathcal{R}}\xspace}
\newcommand{\calD}{\ensuremath{\mathcal{D}}\xspace}
\newcommand{\calP}{\ensuremath{\mathcal{P}}\xspace}

\newcommand{\sAttract}{\ensuremath{s^{\text{attractor}}_i}\xspace}
\newcommand{\sStart}{\ensuremath{s_{\text{start}}\xspace}}
\newcommand{\sGoal}{\ensuremath{s_{\text{goal}}\xspace}}

\DeclareMathOperator*{\argmin}{arg\,min}

\newtheorem{lemma}{Lemma}
\newtheorem{definition}{Definition}
\newtheorem{cor}{Corollary}

\title{Provable Indefinite-Horizon Real-Time Planning for Repetitive Tasks}
\author{
Fahad Islam,
Oren Salzman {\normalfont and}
Maxim Likhachev
\thanks{This research was in part sponsored by ARL, under the Robotics CTA program grant W911NF-10-2-0016.
}
\\
The Robotics Institute, Carnegie Mellon University\\
\{fi,osalzman\}@andrew.cmu.edu,
maxim@cs.cmu.edu
}

\begin{document}
\maketitle

\begin{abstract} 
In many robotic manipulation scenarios, robots often have to perform highly-repetitive tasks in structured environments e.g. sorting mail in a mailroom or pick and place objects on a conveyor belt.
In this work we are interested in settings where the tasks are similar, yet not identical (e.g., due to uncertain orientation of objects) and motion planning needs to be extremely fast. 
Preprocessing-based approaches prove to be very beneficial in these settings---they analyze the configuration-space offline to generate some auxiliary information which can then be used in the query phase to speedup planning times. 
Typically, the tighter the requirement is on query times the larger the memory footprint will be. In particular, for high-dimensional spaces, providing real-time planning capabilities is extremely challenging.
While there are planners that guarantee real-time performance by limiting the planning horizon, we are not aware of general-purpose planners capable of doing it for indefinite horizon (i.e., planning to the goal).
To this end, we propose a preprocessing-based method that provides \emph{provable} bounds on the query time while incurring only a small amount of memory overhead in the query phase.
We evaluate our method on a 7-DOF robot arm and show a speedup of over tenfold in query time when compared to the PRM algorithm.
\end{abstract}

\section{Introduction}
We consider the problem of planning robot motions for highly-repetitive tasks while ensuring bounds on the planning times. There exists manipulation domains where the planning takes non-trivial amount of time, despite the fact that the scenarios are well-structured. Specifically we consider the settings where the environment does not change, the start and goal in each task are similar, yet not identical to the start and goal in previous tasks.

As a running example, consider the problem of mail-sorting where a robot has to put envelopes into appropriate bins (see Fig.~\ref{fig:PR2}). The newly-inserted envelopes do not present any new clutter, and therefore the domain is really static. This domain is being actively pursued in both industry (see for example~\cite{DBot}) and in academia (e.g.,~\cite{hwang2015lazy}).

Another well-suited domain for our planner is a manipulator working at a conveyor belt where it has to pick up objects (arbitrarily positioned and oriented) coming on the conveyor (or drop them off onto a conveyor). In such a setting the robot has to deal with one object at a time and other objects do not cause obstructions. As a result, the domain is also static but start/goal configurations may change. An autonomous robot working at a conveyor is also actively being pursued in industry (see again~\cite{DBot}) and in academia (see \cite{cowley2013perception,menon2014motion}).

\begin{figure}[tb]
  \centering
    \includegraphics[width=0.42\textwidth]{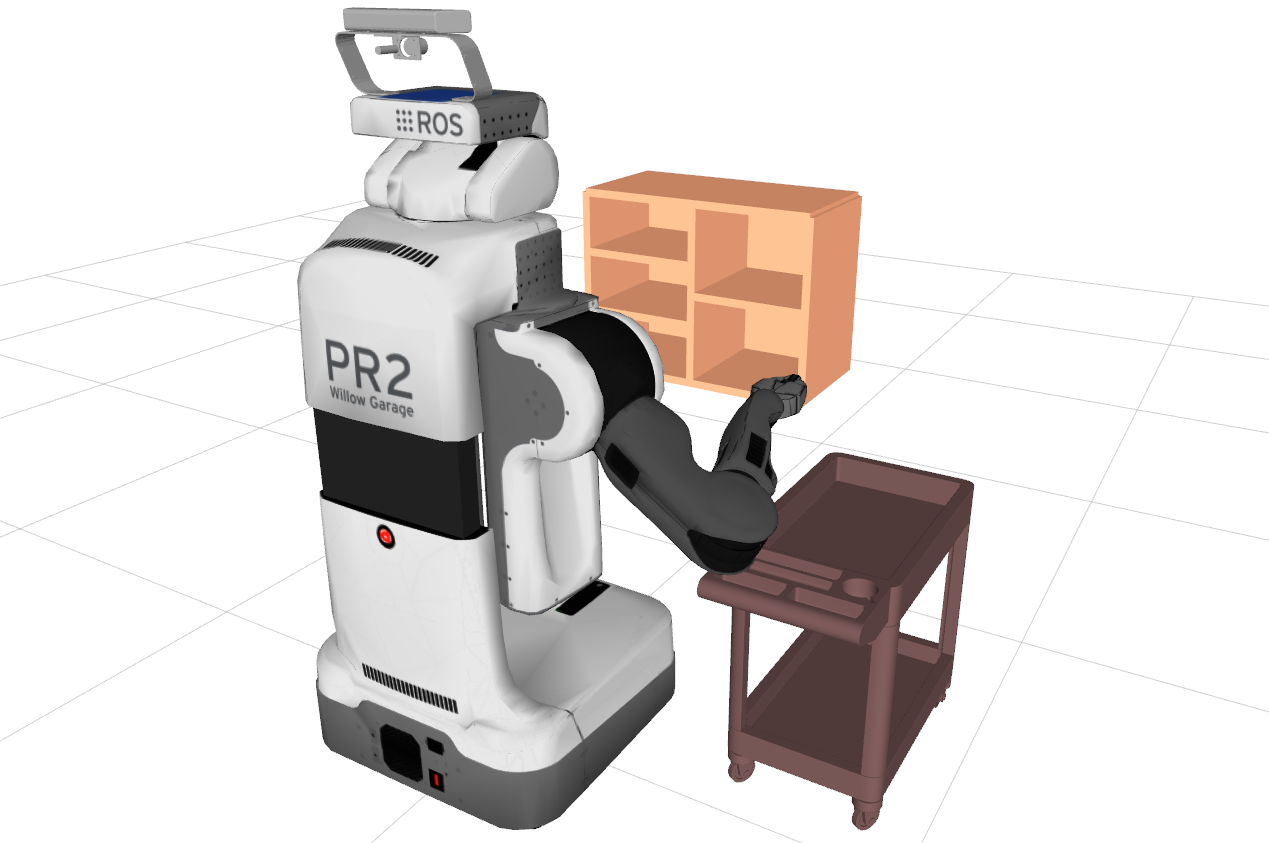}
  \caption{
  Motivating scenario---a robot (PR2) to perform mail-sorting task in a mailroom environment.
}
    \label{fig:PR2}
\end{figure}

Clearly, every time a task is presented to the robot, it can compute a desired path.
However, this may incur large online planning times that may be unacceptable in many settings.
Alternatively, we could attempt to precompute for each start and goal pair a robot path.
However, as the set of possible start and goal locations may be large, caching pre-computed paths for all these queries in advance is unmanageable in high-dimensional configuration spaces\footnote{
A robot configuration is a $d$-dimensional point describing the position of each one of the robot's $d$ joints.
The configuration space of a robot is the $d$-dimensional space of all robot configurations.}.
Thus, we need to balance memory constraints while providing provable real-time query times.

As we detail in Sec.~\ref{sec:rel}, there has been intensive work 
for fast online planning~\cite{LA18} and 
for learning from experience in known environments~\cite{PCCL12,PDCL13,berenson2012robot,CSMOC15}.
Similarly, compressing precomputed data structures in the context of motion-planning is a well-studied problem with efficient algorithms~\cite{SSAH14,DB14}.
However, to the extent of our knowledge, there is no approach that can \emph{provably guarantee} that a solution will be found to \emph{any} query with bounds on planning time and using a small memory footprint.


In this work, we consider the specific case where the start is fixed. Returning to our running example, this corresponds to having a fixed pickup location above the cart (see Fig.~\ref{fig:PR2}).
Our key insight is that given any state $s$, we can efficiently compute a set of states for which a greedy search (to be defined formally in Sec.~\ref{sec:alg}) towards~$s$ is collision free.
Importantly, the runtime-complexity of such a greedy search is bounded and there is no need to perform computationally-complex collision-detection operations. 
This insight allows us to generate in an offline phase a small set of so-called ``attractor vertices'' together with a path between each attractor vertex and the start.
In the query phase, a path is generated by performing a greedy search from the goal to an attractor state followed by the precomputed path from the start.
We describe our approach in Sec.~\ref{sec:alg} and analyze it in Sec.~\ref{sec:analysis}.

We evaluate our approach in Sec.~\ref{sec:eval} in simulation on the PR2 robot\footnote{~\url{http://www.willowgarage.com/pages/pr2/overview}} (see Fig.~\ref{fig:PR2}).
We demonstrate a speedup of over tenfold in query time when compared to the PRM algorithm with a memory footprint of less than 8 Mb while guaranteeing a maximal query time of less than 3 milliseconds (on our machine).


\section{Related work}
\label{sec:rel}
A straightforward approach to efficiently preprocess a known environment is using the PRM algorithm~\cite{kavraki1996probabilistic} which generates a \emph{roadmap}\footnote{A roadmap is a graph embedded in the configuration space where vertices correspond to configurations and edges correspond to paths connecting close-by configurations.}.
Once a a dense roadmap has been pre-computed, any query can be efficiently answered online by connecting the start and goal to the roadmap.
Query times can be significantly sped up by further preprocessing the roadmaps using landmarks~\cite{paden2017landmark}.
Unfortunately, there is no guarantee that a query can be connected to the roadmap as PRM only provides \emph{asymptotic} guarantees~\cite{KKL98}.
Furthermore, this connecting phase requires running a collision-detection algorithm which is typically considered the computational bottleneck in many motion-planning algorithms~\cite{L06}.

Recently, the repetition roadmap~\cite{LA18} was suggested as a way to extend the PRM for the case of multiple highly-similar scenarios.
While this approach exhibits significant speedup in computation time, it still suffers from the previously-mentioned shortcomings.

A complementary approach to aggressively preprocess a given scenario is by minimizing collision-detection time.
However this requires designing robot-specific
circuitry~\cite{MFQSK16}
or limiting the approach to standard manipulators~\cite{YMILV18}.

An alternative approach to address our problem is to precompute a set of complete paths into a library and given a query, attempt to match complete paths from the library to the new query~\cite{berenson2012robot,jetchev2013fast}.
Using paths from previous search episodes (also known as using experience) has also been an active line of work~\cite{PCCL12,PDCL13,berenson2012robot,CSMOC15}.
Some of these methods have been integrated with sparse motion-planning roadmaps (see e.g.,~\cite{SSAH14,DB14}) to reduce the memory footprint of the algorithm.
Unfortunately, none of the mentioned algorithms provide bounded planning-time guarantees that are required by our applications.

Our work bears resemblance to previous work on 
subgoal graphs~\cite{UK17,UK18} and to real-time planning~\cite{KL06,KS09,K90}.
However, in the former, the entire configuration space is preprocessed in order to efficiently answer queries between \emph{any} pair of states which deems it applicable only to low-dimensional spaces (e.g., 2D or 3D).
Similarly, in the latter, to provide guarantees on planning time the search only looks at a finite horizon, generating a partial plan,  and interleaves planning and execution.

Finally, our notion of attractor states is similar to control-based methods that  ensure safe operation over local regions of the free configuration space~\cite{CRC03,CCR06}.
These regions are then used within a high-level motion planner to compute collision-free paths.

\section{Algorithm Framework}
\label{sec:alg}
In this section we describe our algorithmic framework. We start (Sec.~\ref{sec:pdef}) by formally defining our problem and continue (Sec.~\ref{sec:key}) by describing the key idea that enables our approach.
We then proceed (Sec.~\ref{subsec:alg}) to detail our algorithm and conclude with implementation details (Sec.~\ref{subsec:impl}).

\subsection{Problem formulation and assumptions}
\label{sec:pdef}
Let $\calX$ be the configuration space of a robot operating in a static environment containing obstacles.
We say that a configuration is valid (invalid) if the robot, placed in that configuration does not (does) collide with obstacles, respectively.
We are given in advance a start configuration~$\sStart \in \calX$ and some goal region~$G \subset \calX$.
We emphasize that the goal region may contain invalid configurations.
In the query phase we are given multiple queries $(\sStart, s_{\text{goal}})$ where $s_{\rm goal} \in G$ is a valid configuration and for each query, we need to compute a collision-free path connecting $\sStart$ to $s_{\text{goal}}$.

Coming back to our motivating example of mail sorting---the start configuration would be some predefined configuration above the cart where the robot can pick up the envelopes from and the goal region would comprise of all possible placements of the robot's end effector in the cubbies. The environment is static as the only obstacles in the environment are the shelves and the cart which remain stationary in between queries.

We discretize $\calX$ into a state lattice $\calS$ such that any state~$s \in \calS$ is connected to a set of successors and predecessors via a mapping Succs/Preds: $\calS \rightarrow 2^\calS$.
Define $G_\calS := \calS \cap G$ to be the states that reside in the goal region. Note that although our approach is applicable to general graphs (directed or undirected), to be able to reuse the planned path in the reverse direction (e.g. in our motivating example for the motion from the shelve to the start configuration) the graph needs to be undirected.
We make the following assumptions:

\begin{enumerate}[label={\textbf{A\arabic*}}]
  \item \label{assum:1} $G_\calS$ is a relatively small subset of $S$. Namely, it is feasible to exhaustively iterate over all states in $G_\calS$.
However, storing a path from $\sStart$ to each state in $G_\calS$ is infeasible.
  
  \item \label{assum:2} The planner has access to a heuristic function $h: \calS \times \calS \rightarrow \mathbb{R}$ which can estimate the distance between any two states in $G_\calS$. Moreover, 
 \begin{itemize}
  \item The heuristic function should be \textit{weakly-monotone} with respect to $G_\calS$, meaning that $\forall s_1, s_2  \in G_\calS$ where $s_1 \neq s_2 $, it holds that,
  \begin{center}
    $h(s_1, s_2) \geq \min\limits_{s_1' \in \text{Preds}(s_1)} h(s_1', s_2)$.
  \end{center}

  \item The heuristic function $h$ should induce that the goal region is \emph{convex} with respect to $h$, meaning that $\forall s_1, s_2  \in G_\calS$ where $s_1 \neq s_2 $, it holds that,
  \begin{center}
     $\argmin\limits_{s_1' \in \text{Preds}(s_1)} h(s_1', s_2) \in G_\calS$.
  \end{center}

 \end{itemize}
 Namely, for any distinct pair of states ($s_1, s_2$) in $G_\calS$, at least one of $s_1$'s predecessors has a heuristic value less than or equal to its heuristic value.
 Moreover, the predecessor with minimal heuristic value lies in the goal region.

  \item \label{assum:3} The planner has access to a tie-breaking rule that can be used to define a total order \footnote{A total order is a binary relation on some set which is anti-symmetric, transitive, and a convex relation.} over all states with the same heuristic value.
  \end{enumerate}

These assumptions allow us to establish strong theoretical properties regarding the efficiency of our planner. Namely, that
within a known bounded time, we can compute a collision-free path from $\sStart$ to any state in $G_\calS$.

Assumption~\ref{assum:2} may seem too restrictive (especially convexity), imposing that the goal region cannot be of arbitrary structure. However, after we detail our algorithm (Sec.~\ref{subsec:alg}) and analyze its theoretical properties (Sec.~\ref{sec:analysis}), we sketch how we can relax this assumption to be less restrictive.

\subsection{Key idea}
\label{sec:key}
Our algorithm relies heavily on the notion of a greedy search.  Thus, before we describe of our algorithm, we formally define the terms greedy predecessor and greedy search.

\vspace{2mm}
\begin{definition}
\label{def:greedy-suc}
  Let $s$ be some state and $h(\cdot)$ be some heuristic function.
  A state $s' \in \text{Preds}(s)$ is said to be a \emph{greedy} predecessor of $s$ according to $h$ if it has the minimal $h$-value among all of $s$'s predecessors.
\end{definition}
Note that 
if $h$ is weakly monotone with respect to $G_\calS$
(Assumption~\ref{assum:2}) 
and we have some tie-breaking rule
(Assumption~\ref{assum:3}), then every state has a greedy predecessor in the $G_\calS$ and it is unique.
In the rest of the text, when we use the term greedy predecessor, we assume that it is unique and in $G_\calS$.

\vspace{2mm}
\begin{definition}
  Given a heuristic function $h(\cdot)$,
  an algorithm is said to be a \emph{greedy search}  with respect to $h$ if for every state it returns its greedy predecessor according to $h$.
\end{definition}
Note that we define the greedy search in terms of the predecessors and not the successors to account for the directionality of the graph. This will become more clear in Sec.~\ref{subsec:alg}.

\textbf{Remark:} Now that we have the notion of greedy search, we can better explain our definition of convexity with respect to~$h$ (Assumption~\ref{assum:2});
this assumption ensures that a greedy search between any pair of states lies within the goal region, analogously to the standard notion of a convex region where for every pair of points within the region, every point on the straight line segment that joins the pair of points is also within the region.

Our key insight is to precompute in an offline phase subregions within the goal region where a greedy search to a certain (``attractor'') state is guaranteed to be collision free and use these subregions in the query phase.
Specifically, in the preprocessing phase, $G_\calS$ is decomposed into a finite  set of (possibly overlapping) subregions $\calR$.
Each subregion $R_i \in \calR$ is a hyper-ball defined using a center which we refer to as the ``attractor state''~
\sAttract and a radius $r_i$.
These subregions, which may contain invalid states, are constructed in such a way that the following two properties hold
\begin{enumerate}[label={\textbf{P\arabic*}}]
  \item \label{property:1} For any valid goal state $s_{\text{goal}} \in R_i \cap G_\calS$, a greedy search with respect to $h(s, \sAttract)$ over $\calS$ starting at $\sGoal$ will result in a collision-free path to \sAttract.
  \item \label{property:2} The union of all the subregions completely cover the valid states in $G_\calS$. 
      Namely, $\forall s \in G_\calS~\text{s.t.}~s~\text{is valid}, \exists R \in \calR \ s.t. \ s \in R$.
\end{enumerate}

In the preprocessing stage, we precompute a library of collision-free paths $\calL$ which includes a path from $\sStart$ to each attractor state. 
In the query phase, given a query~\sGoal, we 
(i)~identify a subregion $R_i$ such that $\sGoal \in R_i$ (using the precomputed radii~$r_i$),
(ii)~run a greedy search towards~\sAttract by greedily choosing at every point the predecessor that minimizes~$h$ and
(iii)~append this path with the precomputed path in $\calL$ to $\sStart$ to obtain the complete plan.
For a visualization of our algorithm, see Fig.~\ref{fig:approach}.

\begin{figure}
\centering
\includegraphics[width=0.45\textwidth]{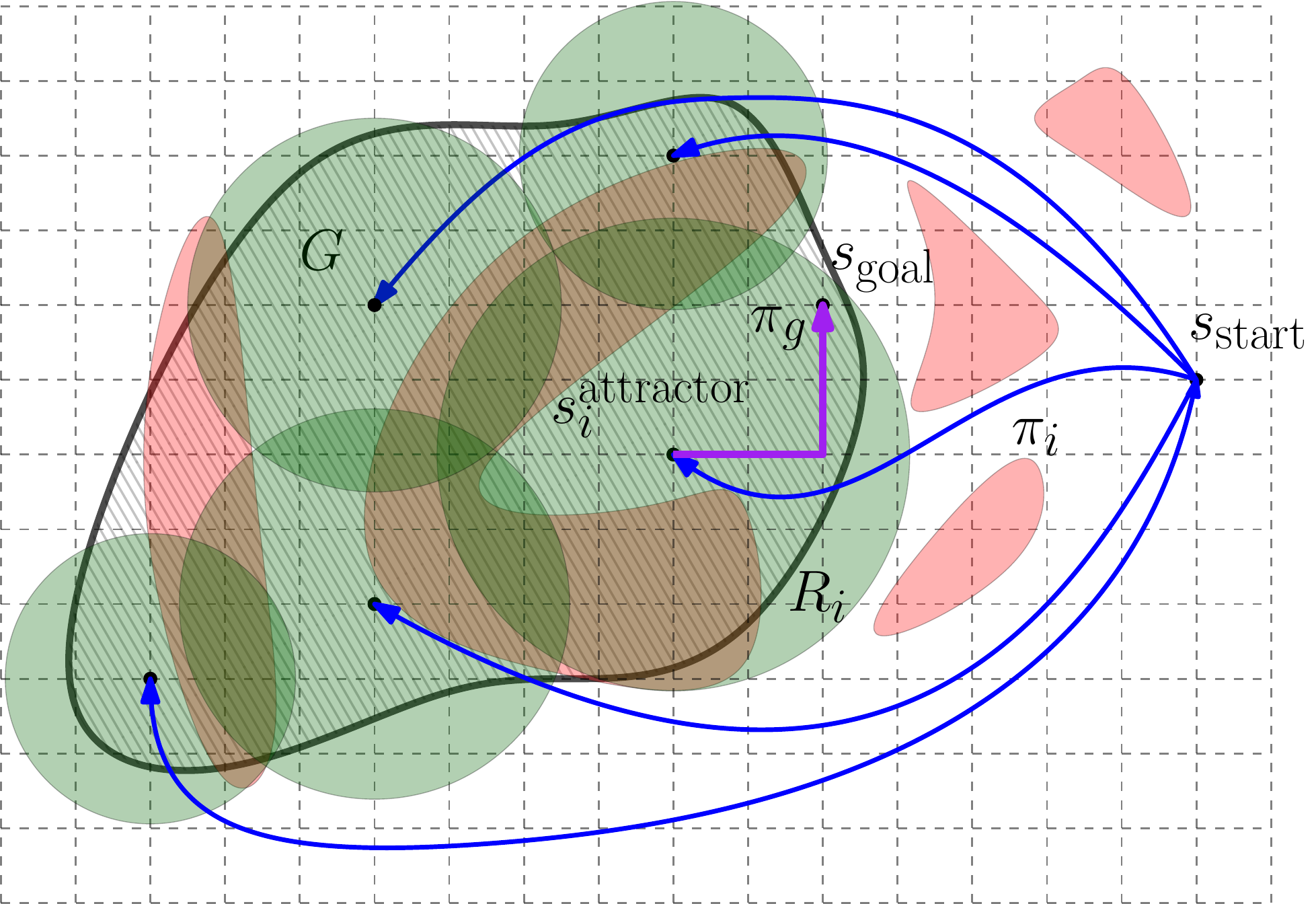}
  \caption{
  Visualization of approach. Subregions are depicted in green, 
  goal region~$G$ is depicted in tiled gray  containing obstacles (red).
  Precomputed paths from $s_{\rm{start}}$ to attractor states are depicted in blue.
 Given a query $(s_{\rm{start}}, s_{\rm{goal}})$, the returned path is given by $\pi_i$ appended with a greedy path ~$\pi_g$ from $s_{\rm{goal}}$ to \sAttract (reversed), which is depicted in purple.
}
    \label{fig:approach}
\end{figure}

\subsection {Algorithm}
\label{subsec:alg}
\subsubsection{Preprocessing Phase}
The preprocessing phase of our algorithm, detailed in Alg.~\ref{alg:1}, takes as input the start state~$\sStart$, the goal region~$G_\calS$ and a conventional motion planner $\calP$, and outputs a set of subregions~$\calR$ and the corresponding library of paths~$\calL$ from~\sStart to each~\sAttract. 

The algorithm covers~$G_\calS$ by iteratively finding a state $s$ not covered\footnote{Here, a state $s$ is said to be covered if there exists some subregion $R \in \calR$ such that $s \in R$.} by any subregion and computing a new subregion centered at $s$.
To ensure that $G_\calS$ is completely covered (Property~\ref{property:2}) we maintain a set~$V$ of valid (collision free) and a set $I$ of invalid (in collision) states called \emph{frontier states} (lines~\ref{alg:1:v} and~\ref{alg:1:i}, respectively).
We also construct subregions centered around invalid states~$\hat{\calR}$ to efficiently store which invalid states have been considered.
We start by initializing~$V$ with some random state in $G_\calS$ and iterate until both~$V$ and~$I$ are empty, which will ensure that $G_\calS$ is indeed covered even if $G_\calS$ is not fully connected.

At every iteration, we pop a state from $V$ (line~\ref{alg:1:pop}), and if there is no subregion covering it, we add it as a new attractor state and compute a path $\pi_i$ from $\sStart$ (line~\ref{alg:1:path}) using the planner $\calP$.
We then compute the corresponding subregion (line~\ref{alg:1:cr} and Alg.~\ref{alg:2}).

As we will see shortly, computing a subregion corresponds to a Dijkstra-like search centered at the attractor state.
The search terminates with the subregion's radius $r_i$ and a list of frontier states that comprise of the subregion's boundary.
The valid and invalid frontier states are then added to $V$ and $I$, respectively (lines~\ref{alg:1:insert_v} and~\ref{alg:1:insert_i}).

Once $V$ gets empty the algorithm starts to search for states which are valid and yet uncovered by growing subregions around invalid states popped from~$I$ (lines~\ref{alg:1:iv_loop}-\ref{alg:1:iv_region} and detailed in Alg.~\ref{alg:3b}). If a valid and uncovered state is found, it is added to $V$ and the algorithm goes back to computing subregions centered at valid states (lines~\ref{alg:1:x_states}-\ref{alg:1:break}), otherwise if $I$ also gets empty, the algorithm terminates and it is guaranteed that each valid state contained in $G_\calS$ is covered by at least one subregion.

\begin{algorithm}[t]
\footnotesize
\hspace*{\algorithmicindent} \textbf{Inputs:} $G_\calS$, $\sStart$, $\calP$
\Comment{goal region, start state and a planner} 

\hspace*{\algorithmicindent} \textbf{Outputs:} 
$\calR, \calL$
\Comment{subregions and corresponding paths to $\sStart$}

\caption{Goal Region Preprocessing}\label{alg:1}
\begin{algorithmic}[1]
\Procedure{PreprocessRegion}{$G_\calS$}
  \State $s \leftarrow$\textsc{ SampleValidState}($G_\calS$)
  \State $V \leftarrow \{ s \}$   \Comment{valid frontier states initialized to random {state}} \label{alg:1:v}
  \State $I$ = $\emptyset$   \Comment{invalid frontier states} \label{alg:1:i}
  \State $ i \leftarrow 0$
       \hspace{2mm} 
       $\calL = \emptyset$
       \hspace{2mm} 
       $\calR = \emptyset$
       \hspace{2mm} 
       $\hat{\calR} = \emptyset$
       
  \vspace{2mm}
    \While {$V$ and $I$ are not empty}
        \While {$V$ is not empty}
          \State $s \leftarrow V.\text{pop}()$ \label{alg:1:pop}
            \If {$\nexists R \in \calR$  s.t. $s \in R$ }  \label{alg:1:discard}
            \Comment{$s$ is not covered}
        \State $\sAttract \leftarrow s$               
        \label{alg:1:attract} 
                \State $\pi_i$ = $\calP.$\textsc{PlanPath}($\sStart, \sAttract$); \label{alg:1:pp}
                \hspace{2mm }
                $\calL \leftarrow \calL \cup \{ \pi_i \}$  \label{alg:1:path}
                \State $(\text{OPEN}, r_i) \leftarrow$ \textsc{ComputeReachability}($\sAttract$) \label{alg:1:cr}
                \State insert Valid(OPEN) in $V$  \label{alg:1:insert_v}
                \State insert Invalid(OPEN) in $I$   \label{alg:1:insert_i}
                \State $R_i$ $\leftarrow$ $(\sAttract, r_i)$; 
                \hspace{2mm} $ i \leftarrow i+1$        \hspace{2mm }
                $\calR \leftarrow \calR \cup \{ R_i \}$
                            \EndIf
        \EndWhile

\vspace{2mm}        
        
        \While {$I$ is not empty} \label{alg:1:iv_loop}
            \State $s$ $\leftarrow$ $I.pop()$
      \If {$\nexists R \in \calR \cup \hat{\calR}$ s.t. $s \in R$ }      \Comment{$s$ is not covered}
\State $(X, r)$ $\leftarrow$ \textsc{FindValidUncoveredState}($s$)
                \State $\hat{R}$ $\leftarrow$ $(s,r)$;
        \hspace{2mm}
        $\hat{\calR} \leftarrow \hat{\calR} \cup \{ \hat{R} \}$ \Comment{invalid subregion}  \label{alg:1:iv_region}
                \If {$X$ is not empty}  \Comment{valid state found}  \label{alg:1:x_states}
                    \State insert $X$ in $V$
                    \State \textbf{break} \label{alg:1:break}
                \EndIf
            \EndIf
        \EndWhile
    \EndWhile

  \vspace{2mm}

  \State \Return $\calR, \calL$
\EndProcedure
\end{algorithmic}
\end{algorithm}

\begin{figure}[tb]
  \centering
    \includegraphics[width=0.3\textwidth]{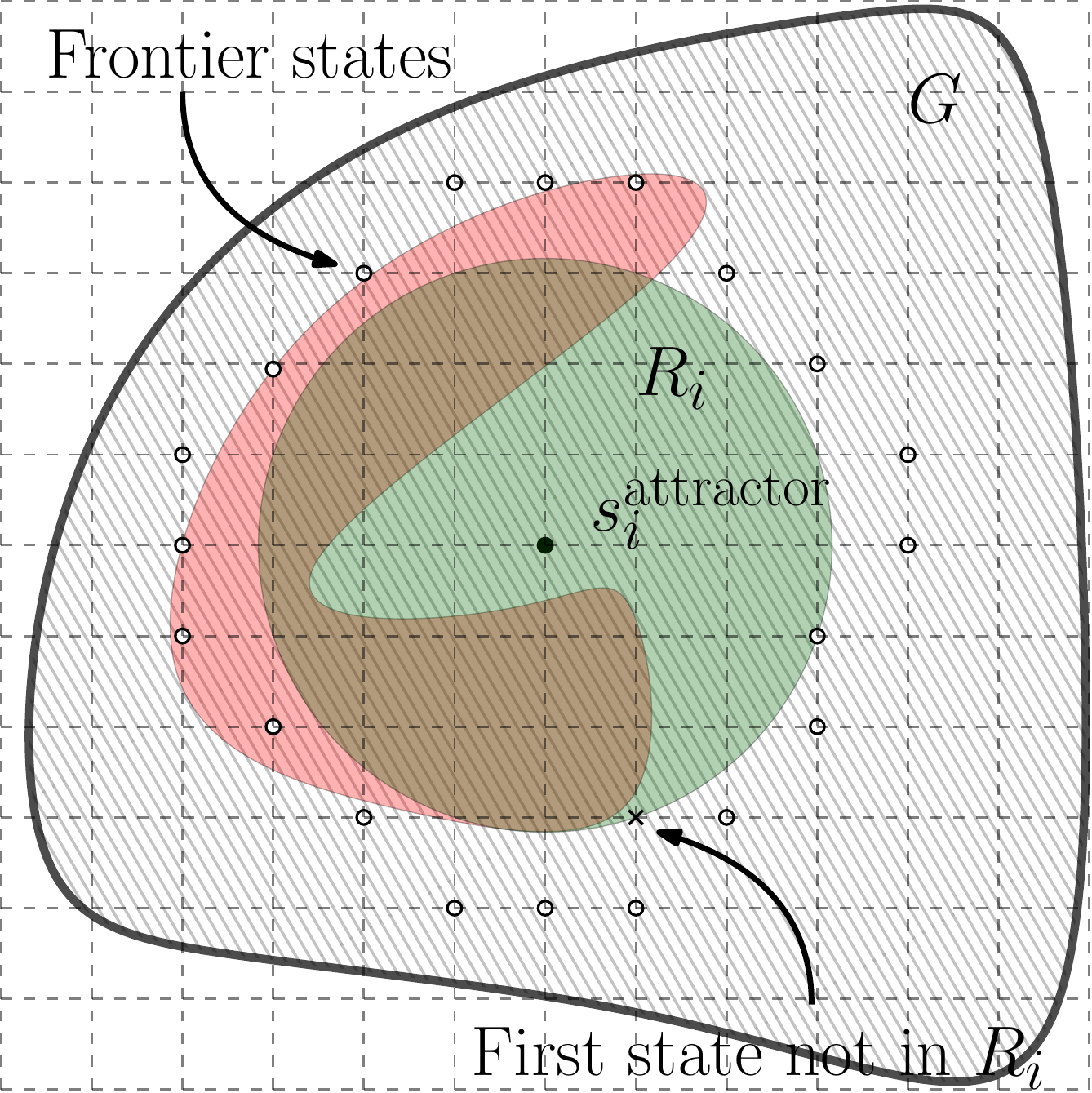}
  \caption{
  Visualization of Alg~\ref{alg:2}. Subregion $R_i$ (green) grown from $\sAttract$ in a goal region~$G$ (tiled grey) containing an obstacle (red).
  Frontier states and first state not in $R_i$ are depicted by circles and a cross, respectively.
}
    \label{fig:alg2}
\end{figure}

\subsubsection{Reachability Search}
The core of our planner lies in the way we compute the subregions (Alg.~\ref{alg:2} and Fig.~\ref{fig:alg2}) which we call a ``Reachability Search''. The algorithm maintains a set of \emph{reachable} states $S_{\text{reachable}}$ for which Property~\ref{property:1} holds.
As we will see, this will ensure that in the query phase, we can run a greedy search from any reachable state $s \in S_{\text{reachable}}$ and it will terminate in the attractor state. 
The following recursive definition formally captures the notion of a reachable state.

\vspace{2mm}
\begin{definition}
  Given some attractor state \sAttract, we say that a state $s \in G_\calS$ is reachable under some function $h(\cdot)$ with respect to \sAttract if either
  (i)~$s = \sAttract$ or
  (ii)~the greedy predecessor of $s$ with respect to $h(s,\sAttract)$ is reachable.
\end{definition}

The algorithm computes a subregion that covers the maximum number of reachable states that can fit into a hyper-ball defined by $h(s,\sAttract)$. 
The search maintains a priority queue OPEN ordered according to $h(s,\sAttract)$. Initially, the successors of $\sAttract$ are inserted in the OPEN (line~\ref{alg:2:OPEN}). For each expanded successor, if its valid greedy predecessor is in $S_{\text{reachable}}$, then the successor is also labeled as reachable (lines~\ref{alg:2:crit} and~\ref{alg:2:set}). 

The algorithm terminates when the search pops a state which is valid but does not have a greedy predecessor state in $S_{\text{reachable}}$ (line~\ref{alg:2:terminate}). Intuitively, this corresponds to  the condition when the reachability search exits an obstacle (see Fig.~\ref{fig:alg2}).
At termination, all the states within the boundary of radius $r_i$ (excluding the boundary) are reachable.

\begin{algorithm}[t]
\footnotesize
\caption{Reachability Search}\label{alg:2}

\begin{algorithmic}[1]
\Procedure{ComputeReachability}{$\sAttract$}
\State $S_{\text{reachable}} \leftarrow \{\sAttract\}$ \Comment{reachable set} \label{alg:2:reachable}
\State OPEN $\leftarrow \{$Succs($\sAttract$)$\}$  \Comment{key: $h(s,\sAttract)$} \label{alg:2:OPEN}
\State CLOSED $\leftarrow \emptyset$
\State $r_i \leftarrow 0$

\While{OPEN$\neq \emptyset$}
    \State $s \leftarrow$ OPEN.pop()
    \State insert $s$ in CLOSED
    \State $s'_g \leftarrow \argmin\limits_{s' \in \text{Preds}(s)} h(s', \sAttract)$ 
  \label{alg:2:greedy} 
 \Comment{greedy predecessor}
    \If {$s'_g$ $\in$ $S_{\text{reachable}}$ and Valid(edge(s,$s'_g$))}  \label{alg:2:crit}
        \State $S_{\text{reachable}} \leftarrow S_{\text{reachable}} \cup \{s\}$  \Comment{$s$ is greedy} \label{alg:2:set}
    \ElsIf {Valid($s$)} \label{alg:2:terminate}
        \State $r_i \leftarrow h(s, \sAttract)$ \label{alg:2:rad}
        \State \Return (OPEN, $r_i$)
    \EndIf
    \For {each $s' \in \text{Succs}(s) \cap G_\calS$} \label{alg:2:prun}
        \If {$s' \notin$ CLOSED}
            \State insert $s'$ in OPEN with priority $h(s', \sAttract)$
        \EndIf
    \EndFor
\EndWhile
\State $r_i \leftarrow h(s, \sAttract) + \epsilon$    \Comment{$\epsilon$ is a small positive constant}
\State \Return (OPEN, $r_i$)

\EndProcedure
\end{algorithmic}
\end{algorithm}

\begin{algorithm}[t]
\footnotesize
\caption{Find valid uncovered state}\label{alg:3b}

\begin{algorithmic}[1]
\Procedure{FindValidUncoveredState}{$\hat{s}$}
\State OPEN $\leftarrow \{\hat{s}\}$
\While{OPEN$\neq \emptyset$}
  \State $s \leftarrow$ OPEN.pop()
  \State insert $s$ in CLOSED
  \If {$\nexists R \in \calR$  s.t. $s \in R$ and Valid(s)}
    \State \Return ($\{s\}$, $h$($s$, $\hat{s}$))
  \EndIf
  \For {each $s' \in \{\text{Succs}(s) \cup \text{Preds}(s)\} \cap G_\calS$}
    \If {$s' \notin$ CLOSED}
        \State insert $s'$ in OPEN with priority $h(s', \sAttract)$
    \EndIf
  \EndFor
\EndWhile
\State $r = h$($s$, $\hat{s})) + \epsilon$    \Comment{$\epsilon$ is a small positive constant}
\State \Return ($\emptyset$, $r$))
\EndProcedure
\end{algorithmic}
\end{algorithm}

\subsubsection{Query Phase}
Given a query goal state $s_{\text{goal}} \in G_\calS$ 
our algorithm, detailed in Alg.~\ref{alg:3}, starts by finding a subregion $R_i \in \calR$ which covers it (line.~\ref{alg:3:covers}). Namely, a subregion~$R_i$ for which 
$h(s_{\text{goal}}, \sAttract) < r_i$.
We then run a greedy search starting from $s_{\text{goal}}$ by iteratively finding for each state $s$ the predecessor with the minimum heuristic $h(s, \sAttract)$ value until the search reaches \sAttract (lines~\ref{alg:3:greedy-call} and~\ref{alg:3:greedy-call-start}-~\ref{alg:3:greedy-call-end}). 
The greedy path~$\pi_g$ is then appended to the corresponding precomputed path~$\pi_i \in \calL$ (line~\ref{alg:3:return}). 
Note that at no point  do we need to perform collision checking in the query phase (given the fact that the environment is static).

\begin{algorithm}[t]
\footnotesize
\caption{Query}\label{alg:3}

\begin{algorithmic}[1]
\Procedure{FindGreedyPath}{$s_1, s_2$}
  \label{alg:3:greedy-call-start}
  \State $s_{\text{curr}} \leftarrow s_2$; \hspace{2mm} $\pi \leftarrow \emptyset$
  \While{$s_{\text{curr}} \neq s_1$}
    \State $\pi \leftarrow \pi \cdot s_{\text{curr}}$
    \Comment{append current state to path}
      \State $s_{\text{curr}} \leftarrow \argmin\limits_{s \in \text{Preds}(s_{\text{curr}})} h(s, s_1)$  \Comment{greedy predecessor}
    \EndWhile
    \State \Return \textsc{Reverse} ($\pi$)   \Comment{reverse to get a path from $s_1$ to $s_2$}
\EndProcedure
\label{alg:3:greedy-call-end}

\vspace{2mm}

\Procedure{Compute path}{$\sGoal$}
  \For {each $R_i \in \calR$}
    \If {$h(s_{\text{goal}}, \sAttract) < r_i$}
    \Comment{$R_i$ covers $\sGoal$}
    \label{alg:3:covers}
      \State $\pi_g \leftarrow$ \textsc{FindGreedyPath} ($\sAttract, \sGoal$)
      \label{alg:3:greedy-call}
      \State \Return $\pi_i \cdot \pi_g$  \Comment{append $\pi_g$ to~$\pi_i \in \calL$}
      \label{alg:3:return}
    \EndIf
  \EndFor

\EndProcedure
\end{algorithmic}
\end{algorithm}

\subsection{Implementation details}
\label{subsec:impl}

\subsubsection{Ordering subregions for faster queries}
Recall that in the query phase we iterate over all subregions to find one that covers \sGoal. 
In the worst case we will have to go over all subregions.
However,  our algorithm typically covers most of the goal region $G_\calS$ using a few very large subregions (namely, with a large radii~$r_i$) and the rest of $G_\calS$ is covered by a number of very small subregions.

Thus, if we order our subregions (offline) according to their corresponding radii, there is a higher chance of finding a covering subregion faster. While this optimization does not change our worst-case analysis (Sec.~\ref{sec:analysis}), it speeds up the query time in case the number of subregions is very large.

\subsubsection{Efficiently constructing $\calL$}
Constructing paths to the attractor states (Alg.~\ref{alg:1}, line~\ref{alg:1:pp}) can be done using any motion planning algorithm $\calP$.
In our implementation we chose to use RRT-Connect~\cite{KL00}.
Interestingly, this step dominates the running time of the preprocessing step.
To improve the preprocessing time, we initially set a small timeout for RRT-Connect to compute a path from an attractor state $\sStart$ to $\sAttract$.
Attractor states for which RRT-Connect fails to find a path to $\sAttract$ are marked as \textit{bad} attractors and we do not grow the subregions from them. 
These, so-called \textit{bad} attractors are discarded when other subregions cover them.

When Alg.~\ref{alg:1} terminates, we reload the valid list $V$ with the remaining bad attractors and rerun Alg.~\ref{alg:1} but this time with a large timeout for RRT-Connect. 
We can also increase the timeout with smaller increments and run Alg.~\ref{alg:1} iteratively until there are no more bad attractors (assuming that there exists a solution for each goal state $\in$ $G_\calS$).

\subsubsection{Pruning redundant subregions}
To reduce the number of precomputed subregions we remove redundant ones after the Alg.~\ref{alg:1} terminates. 
In order to do that, we iterate through all the subregions and remove the ones which are fully contained within any other subregion. 
This step reduces both the query complexity (see Sec.~\ref{sec:analysis}) as well as the memory consumption.

\section {Analysis}
\label{sec:analysis}
In this section we formally prove that 
our algorithm is correct (Sec.~\ref{subsec:correct}) and 
analyze its computational complexity, completeness and bound on solution quality (Sec.~\ref{subsec:complexity},~\ref{subsec:completeness} and ~\ref{subsec:quality} respectively).

\subsection{Correctness}
\label{subsec:correct}
To prove that our algorithm is correct, we show that indeed all states of every subregion are reachable and we can identify if a state belongs to a subregion using its associated radius.
Furthermore, we show that a path obtained by a greedy search within any subregion is valid and that all states in $G_\calS$ are covered by some subregion.
These notions are captured by the following set of lemmas.

\vspace{2mm}
\begin{lemma}
\label{lemma:reachable-1}
Let $S_{\text{reachable}}$ be the set of states computed by Alg.~\ref{alg:2} for some attractor vertex \sAttract.
Every state $s \in S_{\text{reachable}}$ is reachable with respect to \sAttract.
\end{lemma}
\begin{proof}
The proof is constructed by an induction over the states added to $S_{\text{reachable}}$.
The base of the induction is trivial as the first state added to $S_{\text{reachable}}$  is \sAttract (line~\ref{alg:2:reachable}) which, by definition, is reachable with respect to \sAttract.
A state $s$ is added to $S_{\text{reachable}}$ only if its greedy predecessor is in $S_{\text{reachable}}$ (line~\ref{alg:2:greedy}) which by the induction hypothesis is reachable with respect to \sAttract.
This implies by definition that $s$ is reachable with respect to \sAttract.
Note that this argument is true because the greedy predecessor of every state is unique (Assumption~\ref{assum:3}).
\end{proof}

\begin{lemma}
\label{lemma:reachable-2}
Let $S_{\text{reachable}}$ be the set of states computed by Alg.~\ref{alg:2} for some attractor vertex \sAttract.
A state $s$ is in $S_{\text{reachable}}$ iff $h(s, \sAttract) < r_i$.
\end{lemma}

\begin{proof}
Alg.~\ref{alg:2} orders the nodes to be inserted to $S_{\text{reachable}}$ according to $h(s, \sAttract)$ (line~\ref{alg:2:OPEN}).
As our heuristic function is weakly monotonic (Assumption~\ref{assum:2}), the value of $r_i$ monotonically increases as the algorithm adds states to $S_{\text{reachable}}$ (line~\ref{alg:2:rad}).
Thus, for every state $s \in S_{\text{reachable}}$, we have that $h(s, \sAttract) < r_i$.

For the opposite direction, assume that there exists a state $s \notin S_{\text{reachable}}$ such that $h(s, \sAttract) < r_i$.
This may be because  Alg.~\ref{alg:2} terminated due to a node $s'$ that was popped from the open list with 
$h(s', \sAttract) \leq h(s, \sAttract)$.
However, using the fact that our heuristic function is weakly monotonic (Assumption~\ref{assum:2}) we get a contradiction to the fact that $h(s, \sAttract) < r_i$.
Alternatively, this may be because we prune away states that are in $G_\calS$  (Alg.~\ref{alg:2} line~\ref{alg:2:prun}).
However, using the fact that our goal region is convex with respect to $h$ (Assumption~~\ref{assum:2}), this cannot hold.
\end{proof}

\begin{lemma}
\label{lemma:greedy}
Let $R_i \in \calR$ be a subregion computed by Alg.~\ref{alg:2}.
for some attractor vertex \sAttract.
A greedy search with respect to $h(s, \sAttract)$  starting from any valid state $s \in R_i$ is complete and valid.
\end{lemma}

\begin{proof}
Given a state $s \in R_i$, we know that $s \in S_{\text{reachable}}$ (Lemma~\ref{lemma:reachable-2})
and that it is reachable with respect to \sAttract (Lemma~\ref{lemma:reachable-1}).
It is easy to show (by induction) that any greedy search starting at a state $S_{\text{reachable}}$ will only output states in $S_{\text{reachable}}$.
Furthermore, a state is added to $S_{\text{reachable}}$ only if the edge connecting to its greedy predecessor is valid (line~\ref{alg:2:crit}).
Thus, if $s\in S_{\text{reachable}}$ is valid, the greedy search with respect to $h(s, \sAttract)$  starting from $s$ is complete and valid.
\end{proof}

\begin{lemma}
\label{lemma:coverage}
At the end of Alg.~\ref{alg:1}, every state $s \in G_\calS$ is covered by some subregion $R \in \calR$.
\end{lemma}
\begin{proof}
Assume that this does not hold and let $s \in G_\calS$ be a state that is not covered by any subregion but has a neighbor (valid or invalid) that is covered.
If $s$ is valid, then it would have been in the valid frontier states $V$ and either been picked to be an attractor state (line~\ref{alg:1:attract}) or covered by an existing subregion (Alg.~\ref{alg:2}).
A similar argument holds if $s$ is not valid.
\end{proof}

From the above we can immediately deduce the following corollary:

\vspace{2mm}

\begin{cor}
  After preprocessing the goal region~$G_\calS$ (Alg.~\ref{alg:1} and~\ref{alg:2}), in the query phase we can compute a valid path for any valid state $s \in G_\calS$ using Alg.~\ref{alg:3}.
\end{cor}

\textbf{Remark}
We can relax Assumption~\ref{assum:2} in two ways.
The first is by explicitly tracking states not in the goal region instead of pruning them away (Alg.~\ref{alg:2} line~\ref{alg:2:prun}).
Unfortunately, this may require the algorithm to store many states not in $G_\calS$  which may be impractical (recall that Assumption~\ref{assum:1} only states that we can exhaustively store in memory the states in the goal region).
An alternative, more practical way, to relax Assumption~\ref{assum:2}  is by terminating the search when we  encounter a state not in the goal region.
This may cause the algorithm to generate much more subregions (with smaller radii) which may increase the memory footprint and the preprocessing times.

\subsection{Time Complexity of Query Phase}
\label{subsec:complexity}
The query time comprises of 
(i)~finding the containing subregion $R_i$ 
and
(ii)~running the greedy search to $\sAttract$.
Step~(i) requires iterating over all subregions (in the worst case) which takes $O(|\calR|)$ steps while 
step~(ii) requires expanding the states along the path from $\sGoal$ to $\sAttract$ which requires~$O(\calD)$ expansions where $\calD$ is the depth (maximal number of expansions of a greedy search) of the deepest subregion. 
For each expansion we need to find the greedy predecessor, considering at most $b$ predecessors, where $b$ is the maximal branching factor of our graph.
We can measure the depth of each subregion in Alg.~\ref{alg:2} by keeping track of the depth of each expanded state from the root i.e., $\sAttract$. Hence, overall the query phase takes $O(|\calR| + \calD \cdot b)$ operations. The maximal query time can also be empirically profiled after the preprocessing phase.

\textbf{Remark:}
Note that we can also bound the number of expansions required for the query phase by bounding the maximum depth of the subregions. We can do that by terminating Alg.~\ref{alg:2} when the $R_i$ reaches the maximum depth or if the existing termination condition (line ~\ref{alg:2:terminate}) is satisfied.
Having said that, this may come at the price of increasing the number of subregions.

\subsection{Algorithm Completeness}
\label{subsec:completeness}
Our method generates plans by stitching together a path from the library $\calL$ (computed offline using some planner~$\calP$), and a path computed online which is returned by the greedy search on a discretized graph $G_\calS$. For the former, our method simply inherits the completeness properties of the planner $\calP$, whereas for the latter, our method is resolution complete; it follows from the correctness discussion (Section~\ref{subsec:correct}). 

\subsection{Bound on Solution Quality}
\label{subsec:quality}
Let the function $c(\cdot)$ denote the cost of a path and $c^*(\cdot)$ denote the cost of an optimal path. Assuming that we precompute each path $\pi_i \in \calL$ with the optimal cost $c^*(\pi_i)$, from the triangle inequality it can be trivially shown that the quality of complete path $\pi$ computed by our method has an additive suboptimality bound; i.e.,
 $c(\pi) - c^*(\pi) < 2 * c(\pi_g)$,
where $c(\pi_g)$ is the cost of the greedy path from the attractor to the goal state.

\section{Evaluation}
\label{sec:eval}
\begin{table*}[t]
\centering
     \resizebox{1.8\columnwidth}{!}{%
        \begin{tabular}{ l | c c c c c}
           & PRM (4T) & MQ-RRT (4T) & E-graph & RRT-Connect & \textbf{Our method} \\
         \hline
         Planning time [ms]& 21.7 (59.6) & 21.2 (35.5) & 497.8 (9678.5) & 1960 (9652) & \textbf{1.0 (1.6)}\\
         Success rate [$\%$]& 86 & 69.75 & 76.5 & 83.8 & \textbf{100}\\
         Memory usage [Mb] & 1,828 & 225.75 & \textbf{2.0} & - & {7.8}
        \end{tabular}
    }
    \caption{Experimental results comparing our method with other single- and multi-query planners tested on Intel® Core i7‐5600U (2.6GHz) machine with 16GB RAM. The table shows the mean/worst-case planning times, success rates and memory usage for our method and for other multi-query planners preprocessed with quadruple the time that our method takes in precomputation (T = 1,445 seconds). Note that the worst-case time for our method shown in these results ($\sim$1.6 millisecond) is the empirical one and not the computed provable time bound which is 3 milliseconds (on our machine) for this environment.
    Results of sampling-based planners are averaged over 200 uniformly sampled queries. For the sampling-based planners the results were averaged over 4 trials (for the same set of 200 queries) with different random number generation seeds.}
    \label{tab:stats}
\end{table*}

\begin{figure*}[!ht]
 \subfloat[\label{fig:success}]{%
   \includegraphics[width=0.47\textwidth]{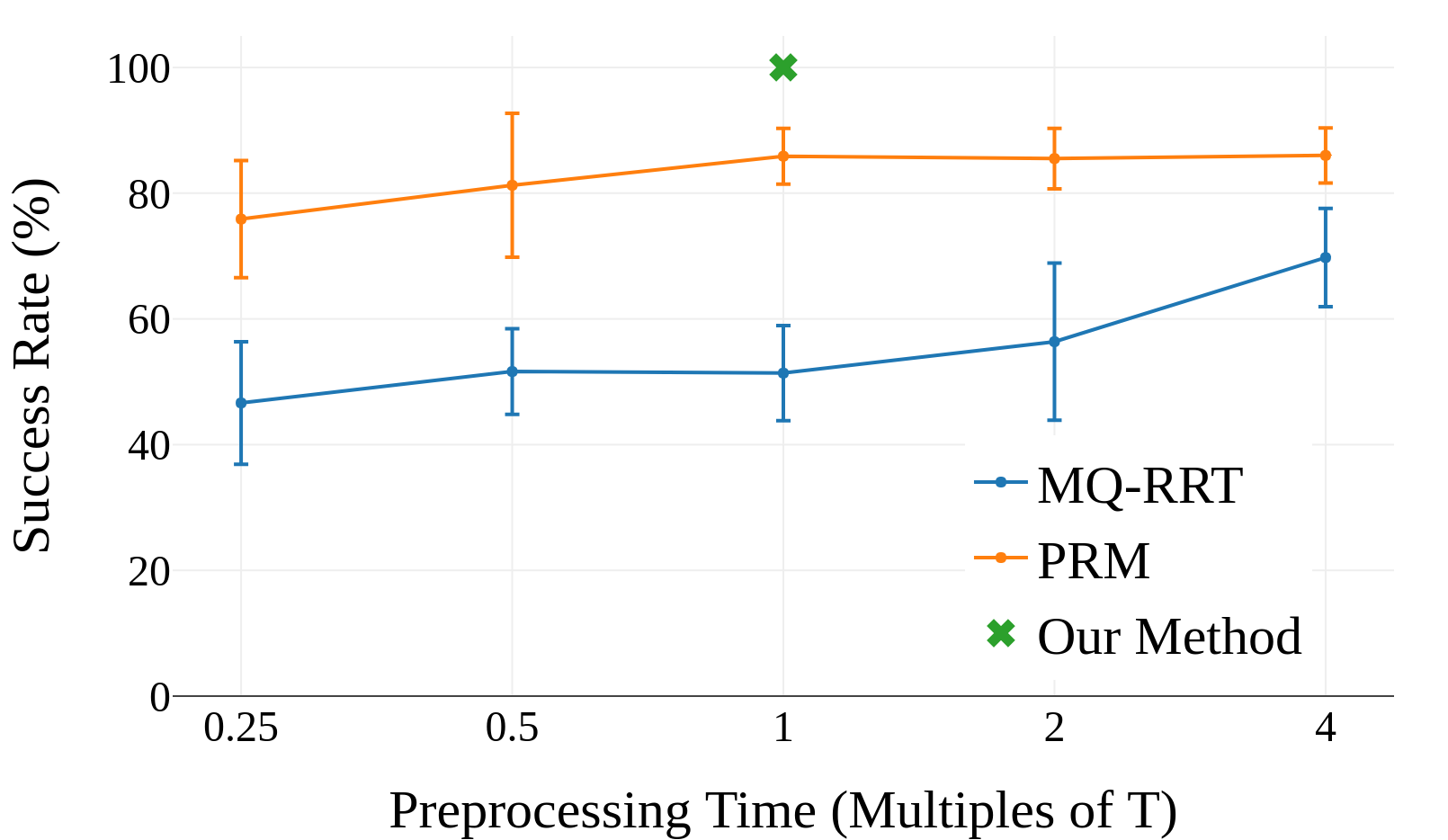}
 }
 \hfill
 \subfloat[\label{fig:memory}]{%
   \includegraphics[width=0.47\textwidth]{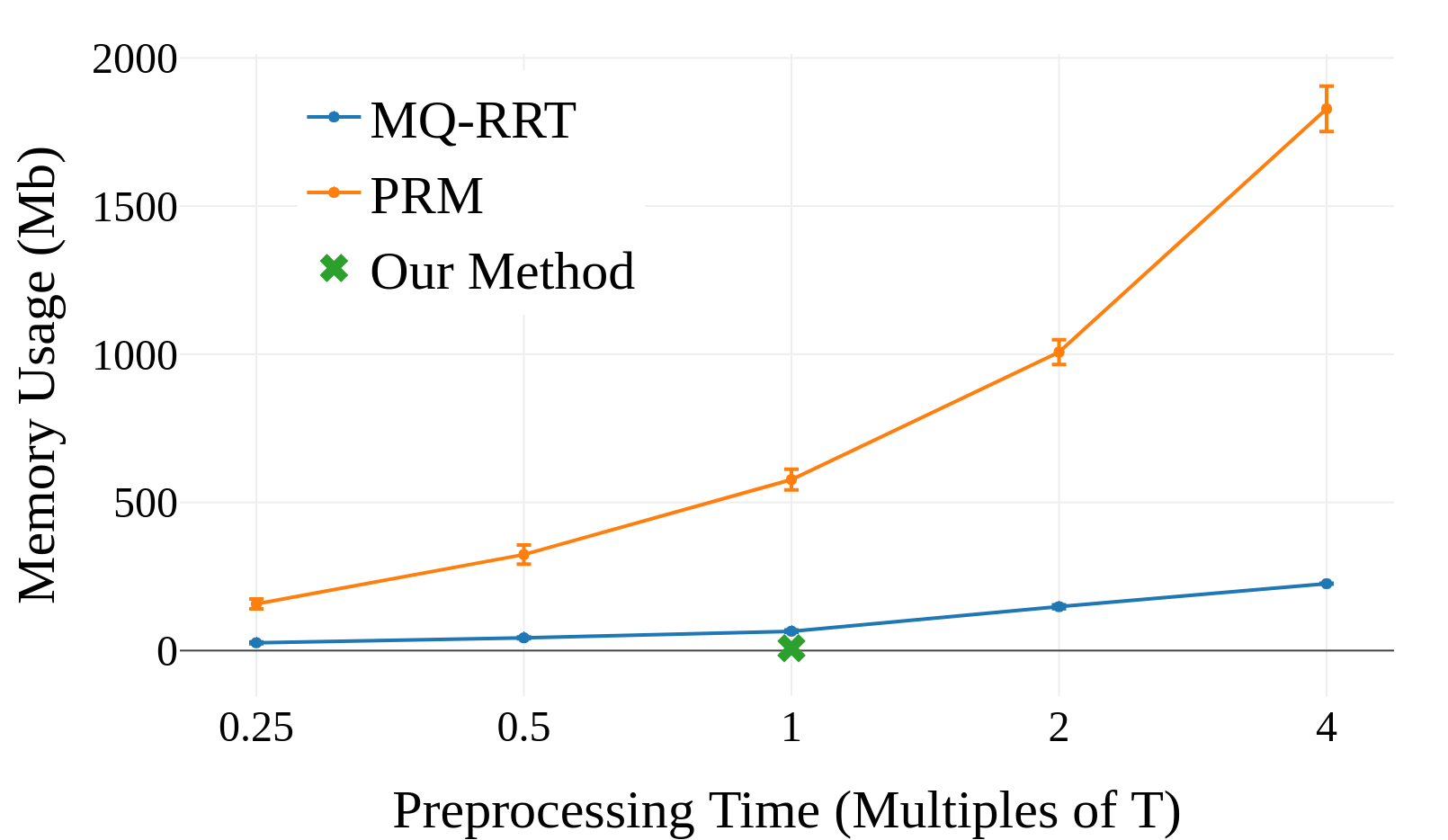}
 }
    \caption{Preprocessing time vs~\subref{fig:success} the success rates and~\subref{fig:memory} memory useage for the 200 queries averaged over 4 trials with different random number generation seeds for the PRM and the MQ-RRT algorithms. The results were computed at the intervals which are multiple of the time T = 1,455 seconds, that our method takes for precomputation. The green cross shows our method for reference.}
    \label{fig:plots}
\end{figure*}

We evaluated our algorithm on the PR2 robot for a single-arm (7-DOF) motion-planning problem. The illustrated task here is to pick up envelopes from a cart and put them in the cubby shelves (see Fig.~\ref{fig:PR2}). Such settings are common in mailroom environments where the robot may have to encounter the same scenario over and over again. The start state is a fixed state corresponding to the pickup location, whereas the task-relevant goal region~$G$ is specified by bounding the position and orientation of the end effector. For this domain, we define~$G$ as a bounding box covering all possible positions and orientations of the end effector within the cubby shelf.

The search is done on an implicit graph $G_\calS$ constructed using $\textit{motion primitives}$ which are small kinematically feasible motions. We define the primitives in the task-space respresentation as small motions for the robot end effector in position axes ($\textit{x, y, z}$) and orientation axes of Euler angles ($\textit{roll, pitch, yaw}$), and a joint-angle motion for the redundant joint of the 7-DOF arm. 
The heurstic function is the Euclidean distance in these seven dimensions. The discretization we use for the graph $G_\calS$ is 2~cm for the position axes, 10 degrees for the Euler axes and 5 degrees for the redundant joint. 
For this domain, we keep the $\textit{pitch}$ and $\textit{roll}$ of the end effector fixed and allow motion primitives along the remaining five dimensions i.e. $\textit{x, y, z, yaw}$ and the reduntant joint angle. We also limit the $\textit{yaw}$ to be between -30 and 30 degrees. Note that the specification of the $G_\calS$ is purely task specific and we exploit the task constraints to limit the size of $G_\calS$ which makes our preprocessing step tractable (Assumption~\ref{assum:1}).

We compared our approach with different single- and multi-query planners in terms of planning times, success rates and memory consumption (see Table~\ref{tab:stats}) for 200 uniformly sampled goal states from $G$. Among the multi-query planners, we implemented PRM, a multi-query version of RRT which we name MQ-RRT and the E-graph planner. 
For MQ-RRT, we precompute an RRT tree rooted at  $\sStart$ offline (similar to PRM) and query it by trying to connect $\sGoal$ to the nearest nodes of the precomputed tree. We use the same connection strategy for MQ-RRT as the one that the asymptotically-optimal version of PRM uses\footnote{In order for the quality of paths obtained the PRM to converge to the quality of the optimal solution, a query should be connected to its $k$ nearest neighbors where 
$k = e(1+1/d)\log(n)$.
Here $n$ is the number of nodes in the tree/roadmap and $d$ is the dimensionality of the configuration space~\cite{karaman2011sampling,SSH16}. 
}.
For PRM, we precomputed the paths from all the nodes in~$G$ to~$s_{\text{start}}$ (this is analogous to our library~$\calL$). 
The query stage thus only required the connect operation (i.e. attempting to connect to $k$ nearest neighbors of $\sGoal$). For both of these planners we also added a goal-region bias by directly sampling from~$G$, five percent of the time\footnote{We used OMPL~\cite{SMK12} for comparisons with the sampling-based planners and modified the implementations as per needed.}. 

For single-query planning, we only report the results for RRT-Connect as it has the fastest run times from our experience. For PRM and MQ-RRT, if the connect operation fails for a query, we considered that case as a failure. For RRT-Connect, we set a timeout of 10 seconds.

For our method, preprocessing (Alg.~\ref{alg:1}) took 1,445 seconds and returned 1,390 subregions. For precomputing paths (Alg.~\ref{alg:1}, line~\ref{alg:1:pp}), we use RRT-Connect. For the first run of Alg.~\ref{alg:1}, we set the timeout to be 10 seconds and for the second run (after reloading $V$ with the bad attractors), we set the timeout to be 60 seconds. By doing so the algorithm finishes successfully with having no remaining bad attractors.

Table~\ref{tab:stats} shows the numerical results of our experiments. Our method being provably guaranteed to find a plan in bounded time, shows a success rate of 100 percent. For the two preprocessing-based planners PRM and MQ-RRT, we report the results for a preprocessing time of 4T (T being the time consumed by our method in preprocessing). The E-graph planner is bootstrapped with a hundred paths precomputed for uniformly sampled goals in $G_S$. For each of these three multi-query planners while running the experiments, the newly-created edges were appended to the auxilary data (tree, roadmap or E-graph) to be used for the subsequent queries.

Among other planners, PRM shows the highest success rate but at the cost of a large memory footprint. The E-graph planner has a small memory footprint but it shows significantly longer planning times. RRT-Connect being a single-query planner happens to be the slowest. Our method shows a speedup of over tenfold in query time as compared to PRM and MQ-RRT and about three orders of magnitude speedup over the E-graph planner and RRT-Connect. The plots in Fig.~\ref{fig:plots} show how the success rate and the memory footprints of PRM and MQ-RRT vary as a function of the preprocessing time. For our domain the PRM seems to saturate in terms of the success rate after T time whereas MQ-RRT continues to improve provided more preprocessing time. In terms of memory usage, PRM's memory footprint grows more rapidly than MQ-RRT.

\section{Conclusion and future work}
We proposed a preprocessing-based motion planning algorithm that provides provable real-time performance guarantees for repetitive tasks and showed simulated results on a PR2 robot. 
Key questions that remain open for future work regard providing stronger guarantees on the solution quality and bounding the number of covering subregions.
We conjuncture that computing the minimal number of subregions is an NP-Hard problem and, if this is the case, we can possibly seek to compute a set of subregions whose size is within some constant-factor approximation of  the size of the optimal set.
Finally, we plan to extend our approach to semi-static environments where there exists both static and non-static obstacles. In such settings, we can also use our planner as an initial pass, having a conventional planner as a backup. We can preprocess our planner with only static obstacles and in the query time we can do a validity check of the computed path; if the path intersects any non-static obstacle, we can fall back to the conventional planner. This can immensely increase the overall throughput of a system if the environment is largely static. 
 
\fontsize{9.3pt}{10.3pt} \selectfont
\bibliographystyle{aaai}
\bibliography{references}

\end{document}